\documentclass[10pt,twocolumn,letterpaper]{article}

\usepackage{iccv}
\usepackage{times}
\usepackage{epsfig}
\usepackage{graphicx}
\usepackage{amsmath}
\usepackage{amssymb}
\usepackage{amsthm}
\usepackage{bm}
\usepackage{mathrsfs,mathtools}
\usepackage{booktabs,color}
\usepackage{subcaption}
\usepackage{helvet}
\usepackage{courier}
\usepackage[hyphens]{url}
\usepackage{appendix}

\newtheorem{definition}{Definition}
\newtheorem{theorem}{Theorem}
\newcommand{\bs}{\mathbf}
\newcommand{\mc}{\mathcal}

\newcommand{\defeq}{\mathrel{\mathop:}=}
\newcommand{\argmin}{\operatornamewithlimits{argmin}}
\newcommand{\subto}{\operatornamewithlimits{s.t.}}

\usepackage[breaklinks=true,bookmarks=false]{hyperref}

\iccvfinalcopy 

\def\iccvPaperID{****} 

\ificcvfinal\fi

\begin{document}

\title{Hierarchical Image Peeling: A Flexible Scale-space Filtering Framework}

\author{Yuanbin Fu\\
Tianjin University\\
{\tt\small yuanbinfu@tju.edu.cn}
\and
Xiaojie Guo\\
Tianjin University\\
{\tt\small xj.max.guo@gmail.com}
\and
Qiming Hu \\
Tianjin University\\
{\tt\small  huqiming@tju.edu.cn}
\and
Di Lin \\
Tianjin University\\
{\tt\small  ande.lin1988@gmail.com}
\and
Jiayi Ma \\
Wuhan University\\
{\tt\small   jyma2010@gmail.com}
\and
Haibin Ling \\
Stony Brook University\\
{\tt\small   hling@cs.stonybrook.edu}
}

\maketitle
\ificcvfinal\thispagestyle{empty}\fi

\begin{abstract}
   The importance of hierarchical image organization has been witnessed by a wide spectrum of applications in computer vision and graphics. Different from image segmentation with the spatial whole-part consideration, this work designs a modern framework for disassembling an image into a family of derived signals from a scale-space perspective. Specifically, we first offer a formal definition of image disassembly. Then, by concerning desired properties, such as peeling hierarchy and structure preservation, we convert the original complex problem into a series of two-component separation sub-problems, significantly reducing the complexity. The proposed framework is flexible to both supervised and unsupervised settings. A compact recurrent network, namely hierarchical image peeling net, is customized to efficiently and effectively fulfill the task, which is about 3.5Mb in size, and can handle 1080p images in more than 60 fps per recurrence on a GTX 2080Ti GPU, making it attractive for practical use. Both theoretical findings and experimental results are provided to demonstrate the efficacy of the proposed framework, reveal its superiority over other state-of-the-art alternatives, and show its potential to various applicable scenarios. Our code is available at \url{https://github.com/ForawardStar/HIPe}.
\end{abstract}
%

\begin{figure}
\centering
	\begin{subfigure}{0.48\linewidth}
		\includegraphics[width=\linewidth]{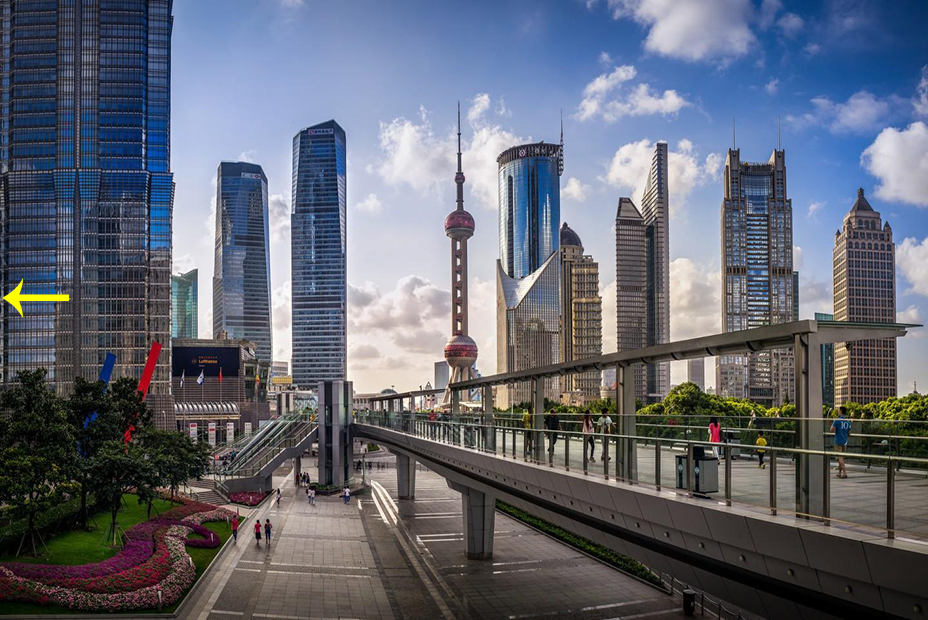}
		\subcaption{Input}
	\end{subfigure}
	\begin{subfigure}{0.48\linewidth}
		\includegraphics[width=\linewidth]{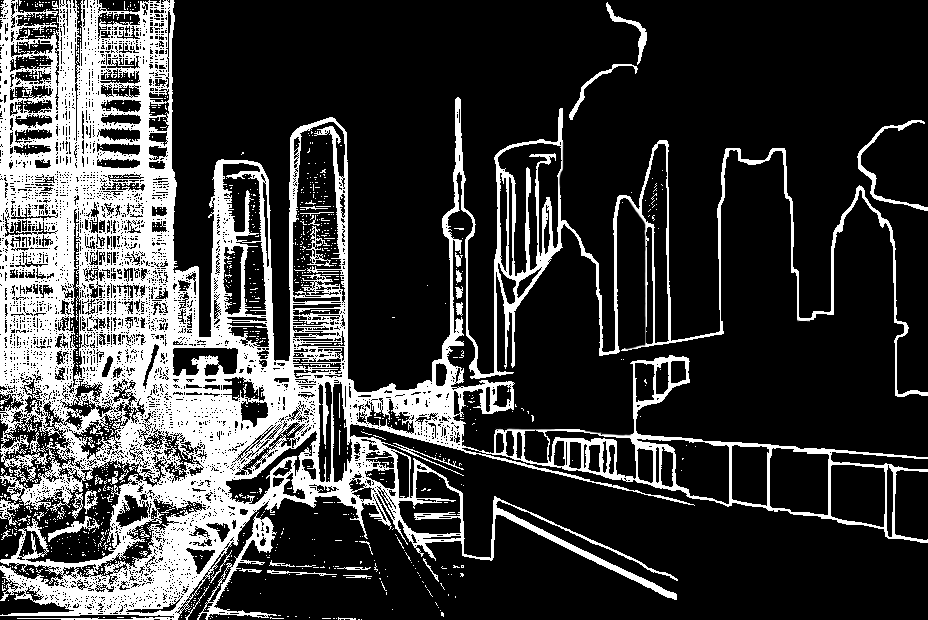}
		\subcaption{Edge guidance}
	\end{subfigure}
	\begin{subfigure}{0.48\linewidth}
		\includegraphics[width=\linewidth]{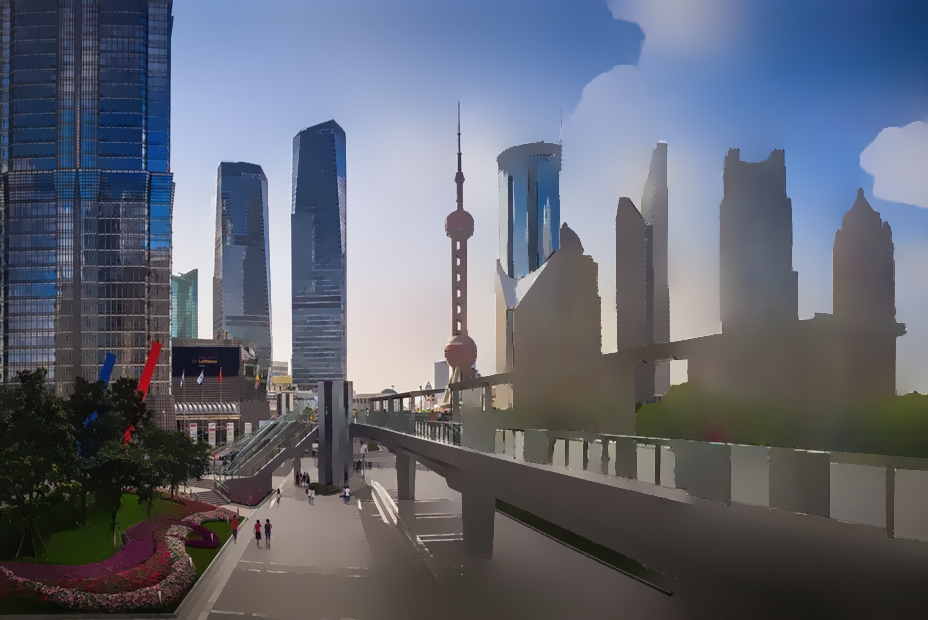}
		\subcaption{ Filtered result}
	\end{subfigure}
		\begin{subfigure}{0.48\linewidth}
		\includegraphics[width=\linewidth]{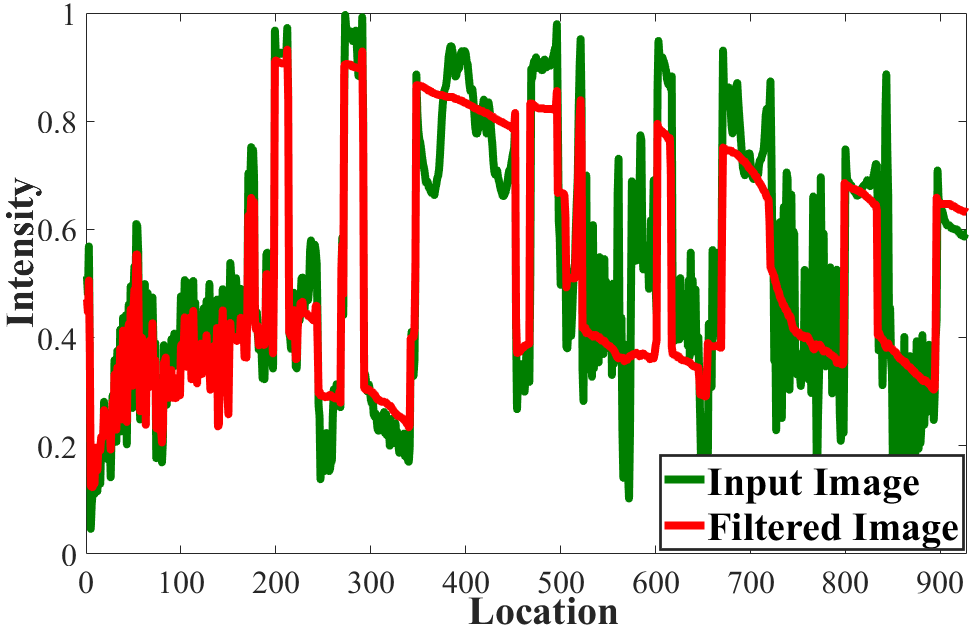}
		\subcaption{1-D signals of intensity}
	\end{subfigure}
	\caption{A peeling example by referring a manually-edited edge guidance with a gradually-changed scale in spatial. (d) corresponds to the row indicated by the yellow arrow in (a).}
	\vspace{-5pt}
		\label{first_EdgeGui}
\end{figure}
\section{Introduction}
Within the past decades, increasing attention to hierarchically organizing images has been drawn from the communities of computer vision and multimedia, by concerning the principle of perceptual systems.  For example, an image can be spatially segmented into a set of object instances or super-pixels \cite{Gao_2019_ICCV,Khan_2019_ICCV,XiaofengAAAI2020,Wang_2020_CVPR,Dual_2020_CVPR,FanZTSX20,FuHZQ20}, which serve as primitives for further processing. Different from the spatial whole-part perspective, this paper concentrates on another organization manner from a scale-space/information eliciting perspective. We call this task {\it image disassembly} for distinguishing it from image segmentation, the formal definition of which is given as follows:
	\begin{definition} [\textbf{Image Disassembly}]
		Given an image $I$, a family of constituent components $\mc{C}\defeq\{C_1,C_2,...,C_n\}$ of $I$ such that $I=\sum_{i=1}^{n}C_i$ is called an image disassembly.
		\label{def:ID}
	\end{definition}
	\noindent The above definition comprises all possible image disassembly strategies. The problem is highly ill-posed since the number of unknowns to recover is $n$ times as many as given measurements. Therefore, we need to impose additional priors or constraints on the desired solution for $C_i$s.

    Since a long time ago, the significance of multi-scale representations of images has been verified, which derives the idea of scale-space filtering. According to human-vision mechanism and scale-space theory, one typically acquires different information at various scales. Simply speaking, a larger scale provides more about structural message of a certain scene, while a smaller one more about textural information. In the literature, a variety of image filtering approaches attempt to separate images into structure and texture components. However, given different images and tasks, hardly a sound way exists for determining which scales are correct or best ones. Furthermore, one may desire a result containing spatially different scales in one case. Please see Figure \ref{first_EdgeGui} and \ref{first_pic} as examples. Rather than eliminating the ambiguity in scale and seeking an optimal image separation, dissembling images in a hierarchical, flexible, and compact fashion is more desired. It should be a very useful feature for practical use in various multimedia, computer vision and graphics applications, such as image restoration \cite{BF2,JIR}, image stylization \cite{Sty,cao2011automatic}, stereo matching \cite{lou2005real,SM06,SM13}, optical flow \cite{OF12,OF15} and semantic flow \cite{SF15,SF16}, where users expect to adjust and select the most satisfactory results.

    For ease of explanation, we respectively denote $I^{t}\defeq\sum_{i=t+1}^nC_i$ and $P^{t}\defeq\sum_{i=1}^tC_i$, such that $I=I^{t}+P^{t}$. Taking the case shown in Fig. \ref{first_pic} as an example, the information left after substracting $t$ of $n$ $C_i$s from the input $I$, say $I^t$, is regarded as the filtering/peeling result. To achieve the above mentioned versatile utility, we advocate the following properties:
	\begin{itemize}
		\item \textbf{Peeling Hierarchy.} For $\forall~~t\in\{1,2,...,n-1\}$, $I^{t}$ should not generate any spurious details passing from $I^{t-1}$ to keep the fidelity of peeling, while more information of $I$ should be contained by $P^{t}$ than $P^{t-1}$.
		\item \textbf{Structure Preservation.} $I^{t}$ should provide a concise yet dominant structure of $I^{t-1}$, and thus of $I$. The edges maintained in $I^{t}$ should be sharp while the rest regions keep as flatten as possible.
		\item \textbf{Flexibility.} In comparison with tuning numerical parameters,  to instruct/constrain the peeling procedure, a more intuitive and flexible way, {\it e.g.} adopting perceptually meaningful cues, might be preferred by users.
		\item \textbf{Model Efficiency.} More and more tasks are expected to be executed in a timely fashion. As a core module of processing, aside from the effectiveness, the efficiency of model is critical.
	\end{itemize}

	\begin{figure}[t]
		\centering
		\includegraphics[width=\linewidth]{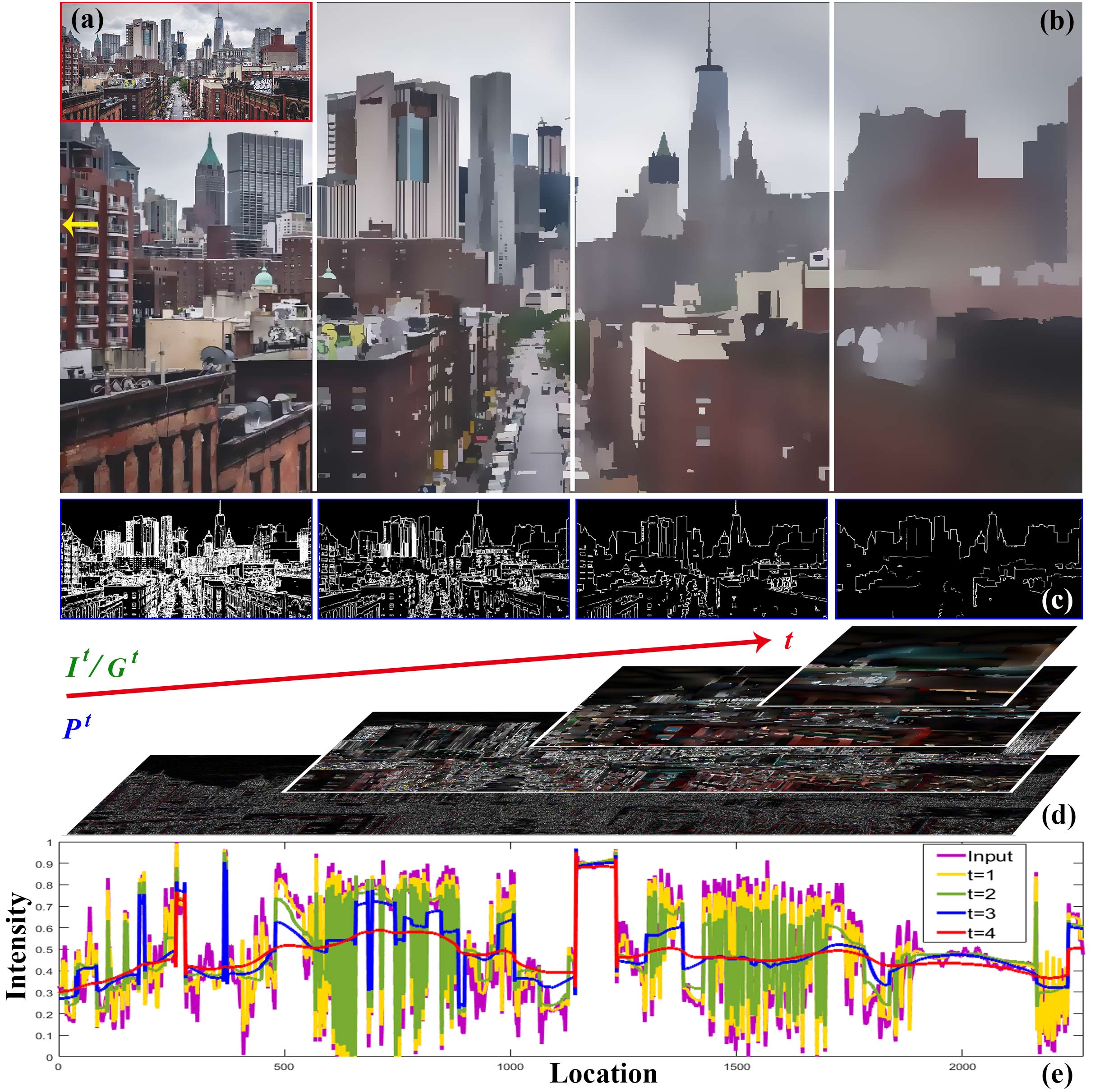}
		\vspace{-10pt}
		\caption{An example of hierarchical image peeling. (b)-(d) respectively show the filtered results $I^t$, estimated guidance maps $G^t$, and peeled information $P^t$ of input image (a) at different scales. (e) depicts the 1-D signals corresponding to the row indicated by the yellow arrow in (b).}
		\vspace{-15pt}
		\label{first_pic}
	\end{figure}

	\subsection{Related Work}
The earliest attempt on scale-space filtering may trace back to \cite{SSF}, which obtains a group of derived images by convolving the original image with different Gaussian kernels. Despite the simplicity, its isotropic nature causes the content blind side. In order to mitigate the isotropic issue, Perona and Malik \cite{anSSF} proposed an anisotropic diffusion method, which tries to keep boundaries sharp and coincided with the semantically meaningful boundaries at each scale. The works along this technical line implicitly or explicitly comprise of two core logical functions: \emph{one is to gradually seek the filtering guidance}, {\it e.g.} Gaussian kernel adjustment in \cite{SSF}, and indicator estimation in \cite{anSSF}, while \emph{the other one executes the filtering based on the determined indicator}.
	
	With the indicator selected by tuning numerical parameters, a large body of research focused on the image filtering function. For instance, the bilateral filter (BF) \cite{BF1,BF2} averages the neighbors of each pixel using the weights calculated by the Gaussian distances in spatial and intensity. 
Though preserving dominant edges, BF frequently produces gradient reversal artifacts and halos \cite{RGA,WLS}. Rolling guidance filter (RGF) \cite{RGF} finds strong edges from the previous iteration to guide the next one, of which the main drawback is the inaccurate edge localization similar to \cite{SSF}. In spite of the mechanism difference between BF and RGF, they are in nature \emph{local operators, making them suffer from the oscillating effect} \cite{muGIF}. To address the local issue, many global approaches have been developed. 
As a representative, 
the $\ell_0$ gradient minimization (L0) \cite{L0} employs the $\ell_0$ norm to constrain preserved edge pixels in the filtered results. But, some isolated spots often appear in their results. The follow-ups attempt to overcome such weakness, by considering different regularizers like \cite{2015L1,semiWLS,RWLS}, and spatial neighborhoods like \cite{RTV}. Recently, some literatures \cite{2017Robust,muGIF} employ two kinds of information to guide the filtering \cite{2017Robust}. 
Although the global methods can generate satisfied results in most cases, they typically require multiple iterations to converge in optimization, each of which involves \emph{computationally expensive operators}, {\it e.g.} the inverse of large matrices, \emph{limiting their applicability to real-time tasks}.

	In comparison, with the emergence of deep neural networks, the processing in the testing phase could be largely accelerated because only feed-forward operations are needed  \cite{deep_joint,DCNN,Recur,generic,2017Fast,2017Deep}. The pioneering work in this category goes to \cite{DCNN}, which trains a network to speed up the procedure. Similarly, the work of \cite{Recur} achieves the acceleration via proposing a hybrid neural network. Both of these two models \emph{can merely output one result closest to the (either gradient or image) reference generated by a target filter with one specific parameter configuration, among multiple candidates with distinct scales}. There exists an unsupervised fashion proposed by \cite{unspuervise}, which however is still unable to produce multiple results of an image with different filtering levels. More recently, several approaches have been developed for providing the option of altering the smoothing strength, without need of retraining the network for different parameter configurations. Specifically, Kim {\it et al.} \cite{2018Structure} alternatively learned a deep solver to the computationally expensive operator in the algorithm,
which allows to vary the smoothing degree by adjusting the parameters. Fan {\it et al.} \cite{fan2019decouple} designed a decoupled algorithm 
to dynamically adjust the weights of a deep network for image operators. Nevertheless, for all the above mentioned traditional and deep-learning methods, \emph{the filtering extent needs to be adjusted, if allowed, by carefully tuning numerical parameters, which is not so intuitive and perceptually meaningful to users in practice}. In addition, none of previous works can handle cases like Fig. \ref{first_EdgeGui}, with spatially-variant scales.

	A simple way for hierarchically disassembling images is to recursively feed the intermediate result from the previous iteration by one existing filter (base) as a new input to the next round with manually tweaked parameters. This manner is undoubtedly cumbersome, inevitably facing to the same shortcomings of the base, {\it e.g.} the requirement of network retraining, the high computational complexity, and the non-intuitive parameter searching for multiple times. Hence, a more faithful strategy is expected, which is our goal.
	
	\subsection{Contribution}
	
	Based on the previous analysis, it can be known that existing scale-space filtering approaches solely achieve one or two of our advocated properties, thus they may not be the best choices for practical use. We aim at building an efficient framework to gradually peel images by simultaneously considering the desired four properties. The main contributions of this paper can be summarized as follows:
    \begin{enumerate}
		\item We define a general image organization problem from a scale-space perspective. We also provide theoretical analysis to simplify the problem, based on which, the original complex problem is converted into a series of sub-problems, significantly reducing the complexity.
		\item We customize a compact recurrent framework, {\it i.e. hierarchical image peeling}, which adopts the perceptually
meaningful cues, {\it i.e. edge map $G^t$}, as the guidance to indicate the peeling procedure (shown in Figure \ref{first_EdgeGui} and \ref{first_pic}) and can be trained with or without paired edge and ground truth filtered image data (supervised or unsupervised training).
		\item Extensive experiments are conducted to demonstrate the efficacy of the proposed framework, reveal its superiority over other state-of-the-art alternatives, and show its potential to various applications.
	\end{enumerate}

	\section{Methodology}

	\subsection{Problem Analysis}
    The main goal of this paper can be expressed by Definition \ref{def:ID} with peeling hierarchy and structure preservation satisfied (\emph{specified feasibility}). We employ the first-order
	derivative of a component, namely $\nabla C_i$, to reflect its detail/structure information. Directly solving such a problem is way too hard due to the complex relationship among multiple components. To make the problem easier to deal with, we simplify it as a series of two-component decomposition sub-problems, based on the following theorem.
	\begin{theorem}[\textbf{Sequential Peeling}]
		Suppose, for any $t$, $[C_t,I^{t}]$ is a feasible solution to separating $I^{t-1}$ into two components. The sequential separation results are also feasible to the original hierarchical image peeling problem.
	\end{theorem}
	\begin{proof}
		According to the peeling hierarchy property, the set of non-zero elements in $\nabla I^{t}$ should be a subset of that in $\nabla I^{t-1}$. Together with the structure preservation, the uncorrelation between $\nabla C_t$ and $\nabla I^{t}$ should be guaranteed, denoted by $\nabla C_t\perp\nabla I^{t}$. Having $I^{t}=I^{t+1}+C_{t+1}$ and $\nabla C_{t+1}\perp\nabla I^{t+1}$ yields $\nabla C_t\perp\nabla I^{t+1}+\nabla C_{t+1}$ and $\nabla C_t\perp\nabla C_{t+1}\perp\nabla I^{t+1}$. In the sequel, both $\nabla C_{1}\perp...\perp\nabla C_{t+1}\perp\nabla I^{t+1}$ and $\nabla P^{t+1}\perp \nabla I^{t+1}$ hold, which establish the claim.
	\end{proof}
	
	The above finding boils down the original image disassembly problem to a sequential processing, naturally motivating us to design a recurrent strategy. Each recurrent unit performs structure preserving image peeling to a controllable extent. Let us here concentrate on the unit function, which can be written usually in the following shape:
	\begin{equation}
	\argmin_{C_{t},I^{t}} \Phi(C_{t})+\alpha\Psi(I^{t})~~ \subto~~ I^{t-1}=I^{t}+C_{t},
	\end{equation}
	where $\Phi(\cdot)$ and $\Psi(\cdot)$ are the regularizers on $C_{t}$ and $I^{t}$ respectively, which can be $\ell_1$ or $\ell_2$ depending on different demands. Moreover, $\alpha$ is a parameter controlling the filtering/peeling strength. \emph{Most of traditional methods}, such as \cite{muGIF,2017Robust,L0,RTV}, on the one hand, \emph{involve computationally expensive operations}, like inverse of large matrices, limiting their applicability. Equipped with deep learning techniques, one can alternatively train a neural network $\mathscr{P}$ from the input-output viewpoint, say $[C_t,I^t]\leftarrow\mathscr{P}(I^{t-1},\alpha)$, to mimic related operators and the whole procedure. Once the network is trained, the inference could be accomplished at cheap expenses.
	On the other hand, \emph{using a numerical parameter $\alpha$ to indicate the peeling extent is not intuitive and difficult to tune}. A perceptually meaningful and faithful guidance is more practical. Among possible cues, edge might be a good candidate to act as the indicator thanks to its simple yet semantic characteristic. Unfortunately, in most situations, the expected edges at each scale/recurrent stage are unknown in advance. What we can do is to seek a ``best" predict of the edges for guiding the generation of $I^{t}$ from the input $I^{t-1}$, {\it i.e.} $G^t\leftarrow\mathscr{G}(I^{t-1})$. Putting the aforementioned concerns together yields a strategy that recurrently solves $[C_t,I^t]\leftarrow\mathscr{P}(I^{t-1},G^t)$. Please notice that our strategy also allows edge maps constructed by users, and can be generalized to other possible types of guidance. We emphasize that none of existing works can produce a filtering result adhering to the edge guidance like the case in Fig. \ref{first_EdgeGui}.
	
	\begin{figure}[t]
		\centering
		\includegraphics[width=\linewidth]{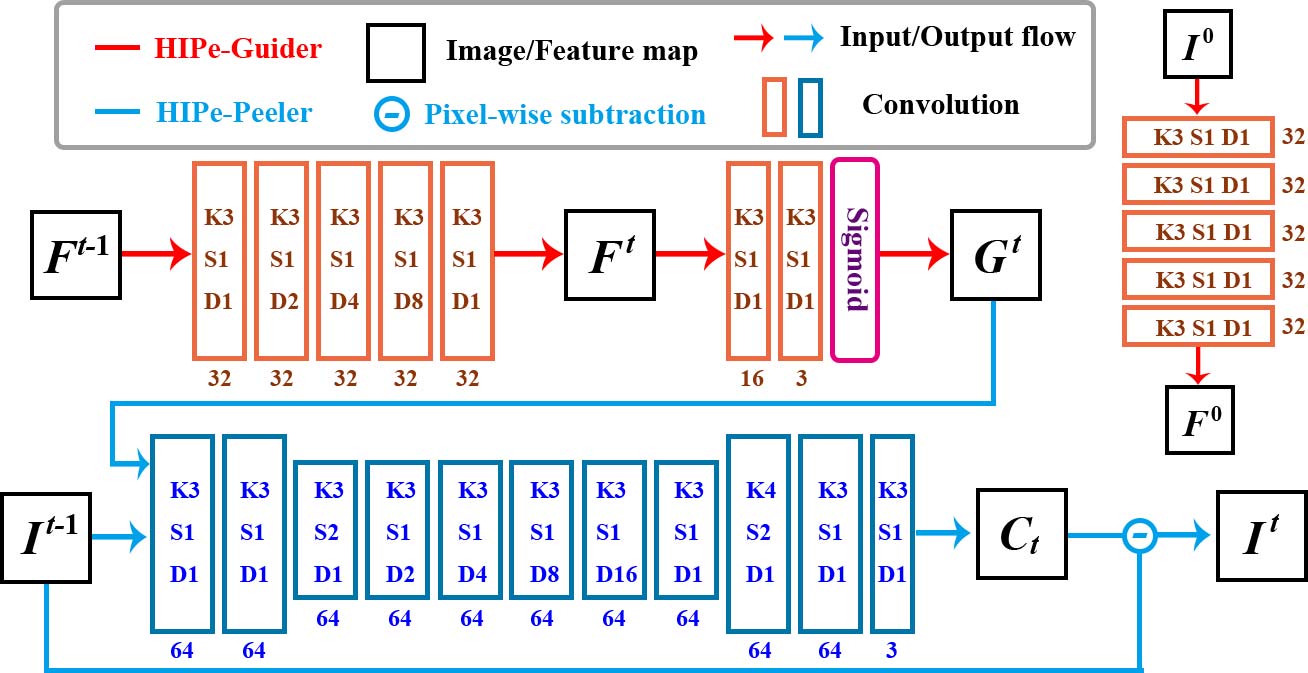}
		\vspace{-10pt}
		\caption{The architecture of proposed HIPe-Net.The number under each block represents how many channels generated, while the letters $K$, $S$ and $D$ indicate the size of kernel, the stride and the dilation rate, respectively.}	
		\vspace{-5pt}
		\label{arch}
	\end{figure}
	
	\subsection{Hierarchical Image Peeling Network}
	Figure \ref{arch} shows the architecture of our hierarchical image peeling network (HIPe-Net for short), which performs in a recurrent fashion and consists of two logical modules. One module $\mathscr{G}$ responds to the guidance prediction, while the other $\mathscr{P}$ takes care of the peeling functionality \emph{conditioned on} the given/predicted guidance.
	By this logical partition, the two modules can be greatly decoupled, thus further simplifying the problem. In addition, both the model reduction and training procedure can benefit from the partition, because the original space is considerably restricted. \\
	
	\textbf{HIPe-Peeler $\mathscr{P}$.} Generally, our framework can not only learn the behavior of any existing filter, like muGIF \cite{muGIF}, RTV \cite{RTV}, and L0 \cite{L0}, but also be trained in an unsupervised manner (no edge and expected filtered image pair). {At each recurrent step, the core mission of the peeler is to produce a result from the input, which should strictly adhere to the guidance}, no matter what the guidance looks like only if reasonable. At the moment, the guidance $G^{t}$ is assumed being already at hand. We will see how to form the guidance shortly. Due to the hard constraint $I^{t-1}=I^{t}+C_t$, the peeler could map the input $I^{t-1}$ only to one component, $C_{t}$ or $I^t$. Then, the other component can be consequently obtained by subtracting the mapped result from the input. For better considering contextual features, the peeler network ought to catch relatively large receptive fields. To this end, we follow \cite{unspuervise}\cite{CAN} by utilizing dilated convolutions to progressively enlarge the receptive field with the dilation rate exponentially increasing, instead of resorting to a deeper network, to cut model size. Please see the peeler part in Figure \ref{arch} for more details.\\
	
	\textbf{HIPe-Guider $\mathscr{G}$.}
	For the input $I^{t-1}$, the expected filtering result can be obtained by setting the parameters of a filter to $\alpha^t$ (for different scales, $\alpha^t$ may vary), {\it i.e.} ${I^{t}_{\mathscr{F}}}\leftarrow\mathscr{F}(I^{t-1},\alpha^t)$ where ${I^{t}_{\mathscr{F}}}$ denotes the ground truth. The main job of the guider is to build the connection between the numerical $\alpha^t$ and its corresponding perceptual guidance (edge in this work). To achieve the goal, the gradient map of ${I^{t}_{\mathscr{F}}}$ can be computed as guidance, namely $\nabla{I^{t}_{\mathscr{F}}}$. Still, there exists a gap between the gradient and the concept of edge. A proper edge map ought to be semantically meaningful and binary to respectively reflect the perceptual subjective, and avoid the ambiguity about whether a pixel is on edge or not.
	For filling the gap, on the one hand, similar to HIPe-Peeler, we introduce the dilation convolutions with increasing dilation rates to learn the contextual features, adopting which sufficiently improves the perception of semantic objects. On the other hand, a Sigmoid layer is then concatenated to enforce the output to be closely binary. Please see the guider part in Figure \ref{arch} for details. However, the ground truth $I^{t}_{\mathscr{F}}$ is not always available during training. Hence, the gradient map of $I^{t}_{\mathscr{F}}$ should be approximated with the gradient of the raw input $I^0=I$. 
By reusing the gradient information of the raw input image, and if possible, its corresponding annotated edge, a series of $\nabla{\hat{I}^{t}_{\mathscr{F}}}$s with different scales can be built to perform as a rational guider.

%
 Note that directly using an off-the-shelf deep edge detection models like \cite{HED}\cite{RCF}\cite{He_2019_CVPR} to produce edge guidance will lead to large parameter space, because these deep learning approaches require the multi-scale features generated from multiple layers of (very) deep pre-trained networks, like AlexNet \cite{AlexNet} or VGG16 \cite{VGG16}, to obtain multi-scale representations \textbf{(different params. for different scales)}. As shown in Figure \ref{arch}, such an issue can be mitigated in our framework, since we reusing the network parameters, and recurrently take the previous output features $F^{t-1}$ as the input again to get the features $F^t$ with larger receptive fields \textbf{(shared params. for different scales)}. Experiments will reveal that our solution can reach promising performance.

	\begin{figure*}[t]
			\centering

		\begin{subfigure}{0.246\linewidth}
			\includegraphics[width=1\linewidth]{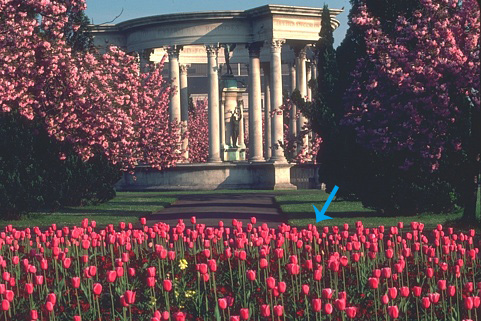}
			\subcaption*{\small{Input}}
		\end{subfigure}
		\begin{subfigure}{0.246\linewidth}
		\includegraphics[width=1\linewidth]{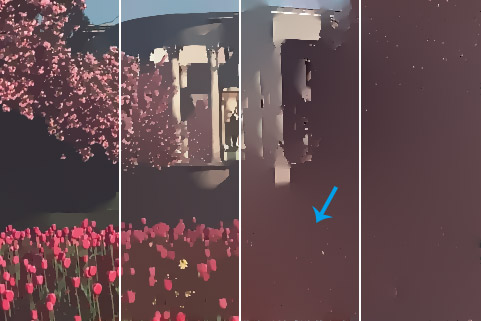}
		\subcaption*{\small{L0}}
	\end{subfigure}
	\begin{subfigure}{0.246\linewidth}
	\includegraphics[width=1\linewidth]{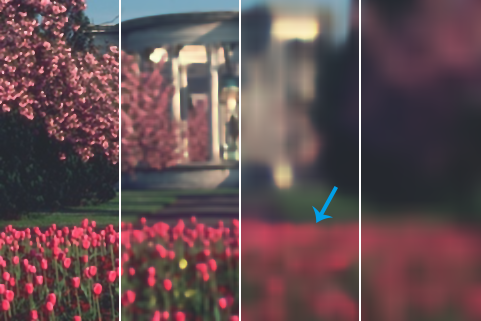}
	\subcaption*{\small{RGF}}
\end{subfigure}
\begin{subfigure}{0.246\linewidth}
	\includegraphics[width=1\linewidth]{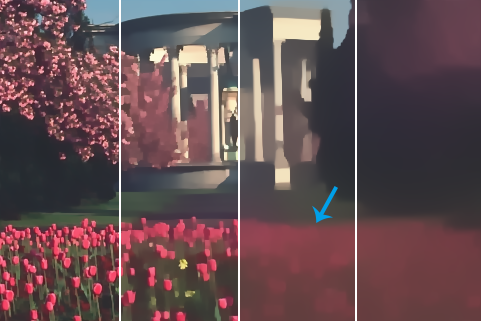}
	\subcaption*{\small{RTV}}
\end{subfigure}\\

		\begin{subfigure}{0.246\linewidth}
	\includegraphics[width=1\linewidth]{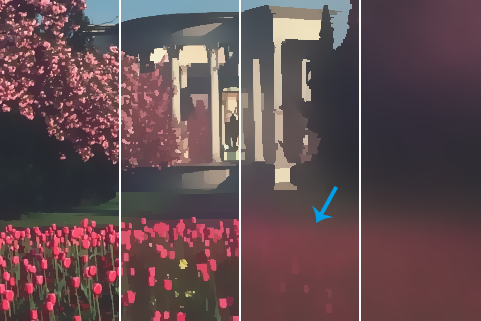}
	\subcaption*{\small{muGIF}}
\end{subfigure}
\begin{subfigure}{0.246\linewidth}
	\includegraphics[width=1\linewidth]{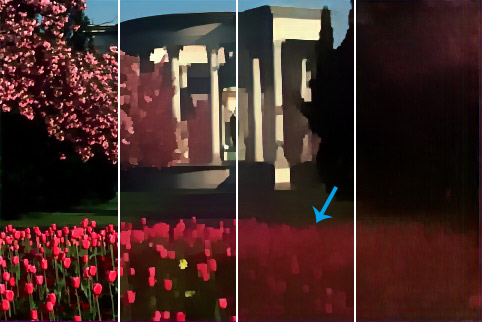}
	\subcaption*{\small{PIO}}
\end{subfigure}
\begin{subfigure}{0.246\linewidth}
	\includegraphics[width=1\linewidth]{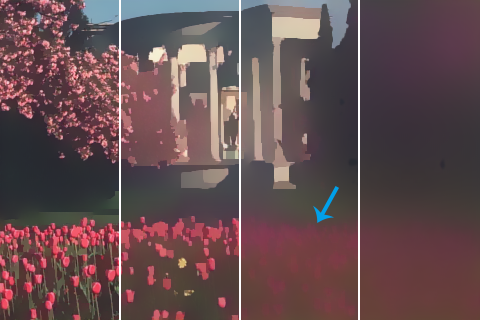}
	\subcaption*{\small{Ours-S}}
\end{subfigure}
\begin{subfigure}{0.246\linewidth}
	\includegraphics[width=1\linewidth]{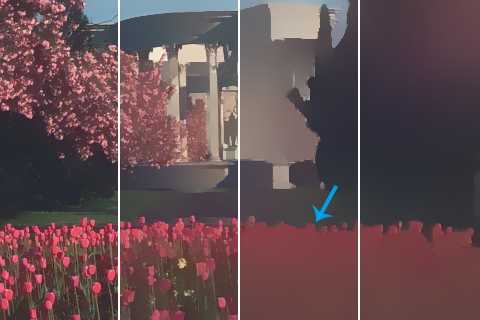}
	\subcaption*{\small{Ours}}
\end{subfigure}
~\\
		\begin{subfigure}{0.246\linewidth}
			\includegraphics[width=1\linewidth]{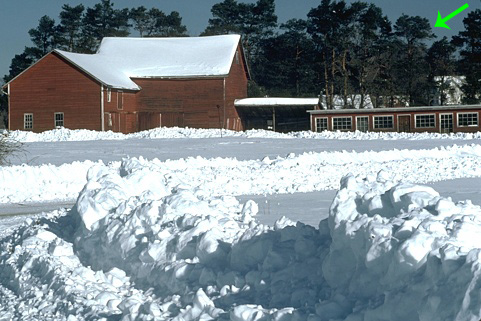}
			\subcaption*{\small{Input}}
		\end{subfigure}
		\begin{subfigure}{0.246\linewidth}
		\includegraphics[width=1\linewidth]{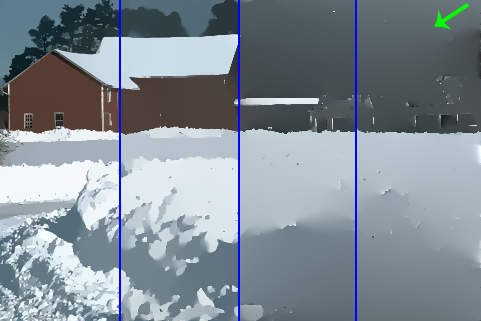}
		\subcaption*{\small{L0}}
	\end{subfigure}
	\begin{subfigure}{0.246\linewidth}
	\includegraphics[width=1\linewidth]{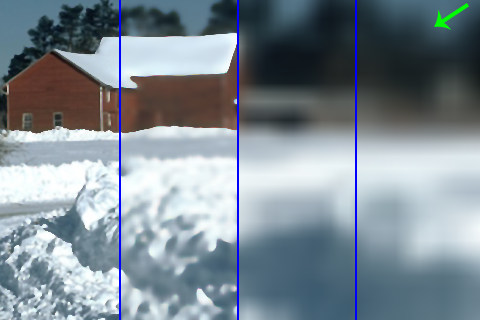}
	\subcaption*{\small{RGF}}
\end{subfigure}
\begin{subfigure}{0.246\linewidth}
	\includegraphics[width=1\linewidth]{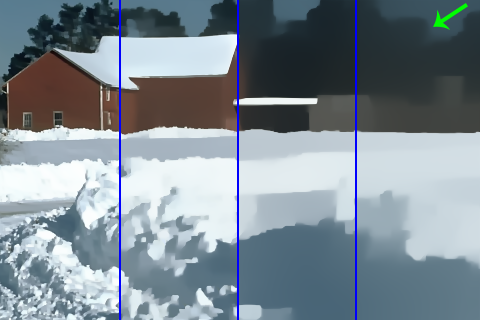}
	\subcaption*{\small{RTV}}
\end{subfigure}\\

		\begin{subfigure}{0.246\linewidth}
	\includegraphics[width=1\linewidth]{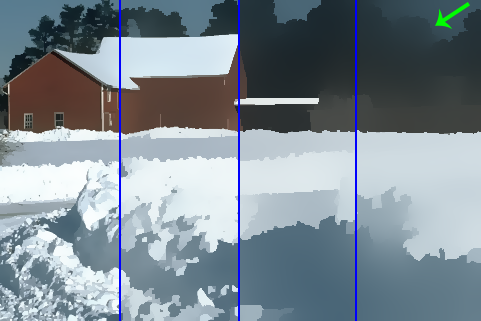}
	\subcaption*{\small{muGIF}}
\end{subfigure}
\begin{subfigure}{0.246\linewidth}
	\includegraphics[width=1\linewidth]{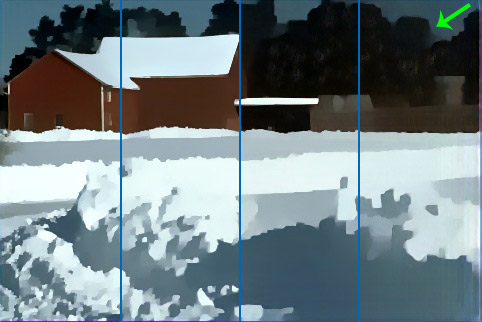}
	\subcaption*{\small{PIO}}
\end{subfigure}
\begin{subfigure}{0.246\linewidth}
	\includegraphics[width=1\linewidth]{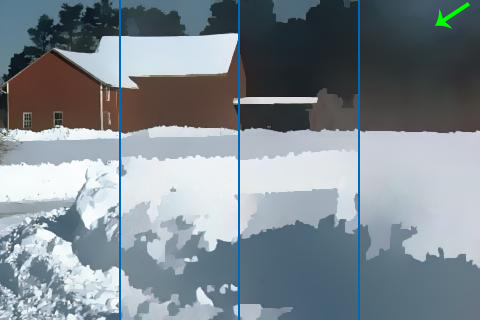}
	\subcaption*{\small{Ours-S}}
\end{subfigure}
\begin{subfigure}{0.246\linewidth}
	\includegraphics[width=1\linewidth]{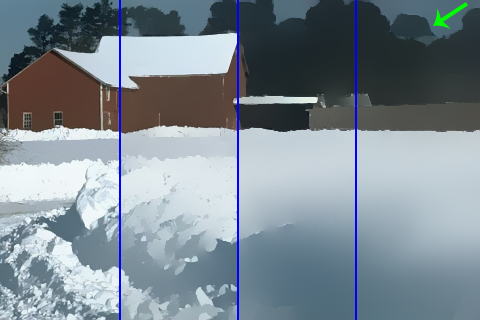}
	\subcaption*{\small{Ours}}
\end{subfigure}
\vspace{-10pt}
\caption{Visual comparison with state-of-the-art methods on scale-space filtering. }
		\label{com}
\centering
		\begin{subfigure}{0.33\linewidth}
	\includegraphics[width=1\linewidth]{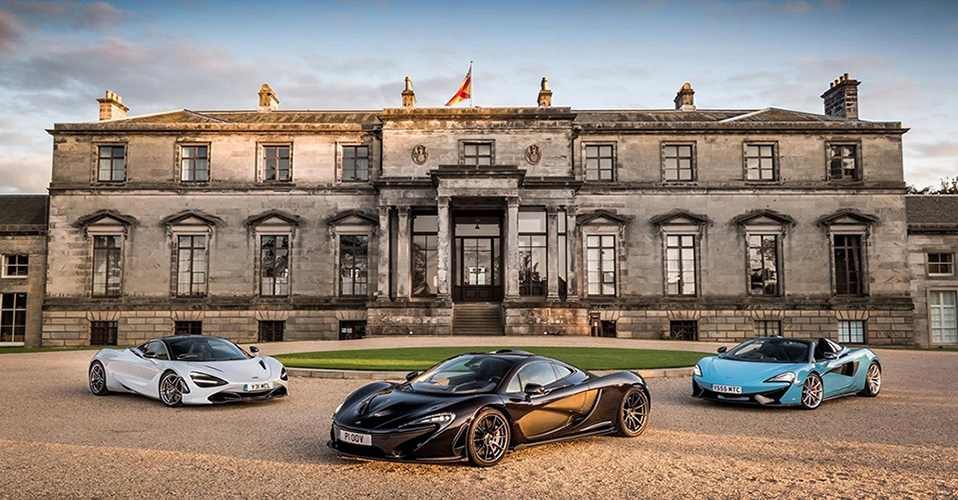}
	\subcaption{Input}
\end{subfigure}
\begin{subfigure}{0.33\linewidth}
	\includegraphics[width=1\linewidth]{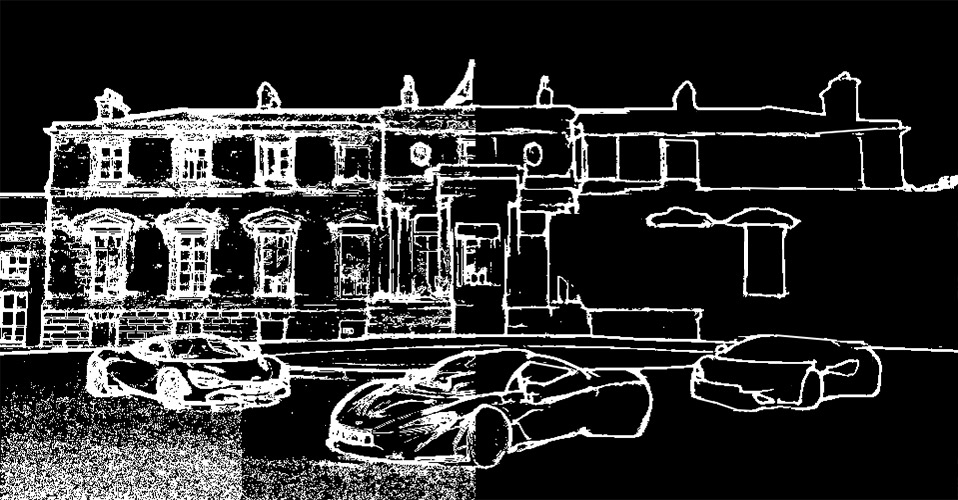}
	\subcaption{Guidance map}
\end{subfigure}
\begin{subfigure}{0.33\linewidth}
	\includegraphics[width=1\linewidth]{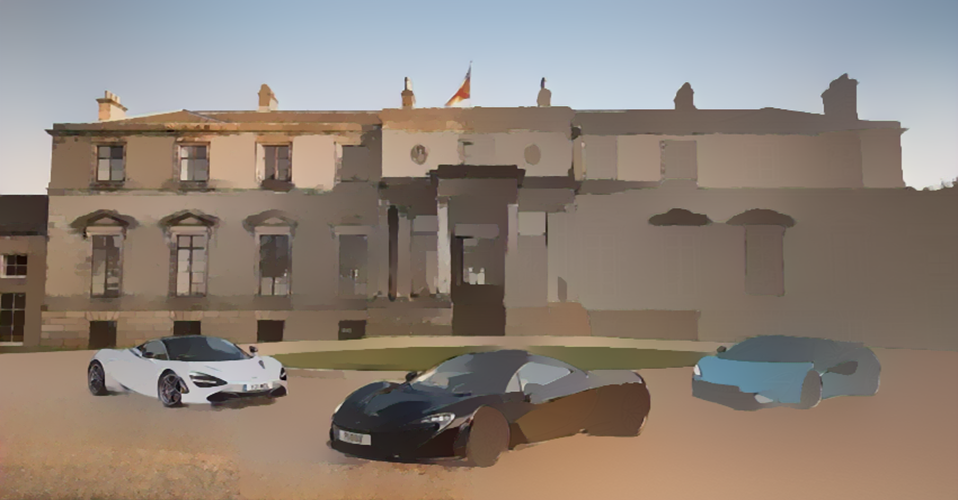}
	\subcaption{Our result}
\end{subfigure}

\vspace{-10pt}
	\caption{Flexibility verification by dealing with a manually-edited guidance map, which exhibits a gradually-changed scale in spatial. Only our HIPe can successfully obtain a peeling result strictly adhering to the guidance map.}
	\label{variededge}
\end{figure*}
	
	\subsection{Network Training}

    For different scales, $I^{t}$ are the processed results by setting different $\alpha^t$s.  The $\alpha^t$ corresponding to $G^t$ gradually increases through setting $\alpha^{t+1}\leftarrow \eta\times\alpha^{t}$ ($t\in\{1,2,...,T\}, \eta \ge 0$). 
    Please notice that the HIPe-Net can recur as many iterations as required, not limited to $T$. Besides, the interval of filtering controlled by $\alpha^t$ in the training phase can also be adjusted according to particular demands.
	
	The objective consists of guider consistency, peeler reconstruction, peeler preservation, and peeler consistency, the formulation of which is uniform for both supervised and unsupervised manners. The only difference is whether the ground truth $I^{t}_{\mathscr{F}}$ is available. For the supervised peeling, $I^{t}_{\mathscr{F}}$ and $\nabla{I^{t}_{\mathscr{F}}}$ are procurable where $I^{t}_{\mathscr{F}}$ is allowed to be generated by any existing filter. For the unsupervised peeling, the ground truth $I^t_{\mathscr{F}}$ is unavailable, thus we have $I^{t}_{\mathscr{F}}\defeq I^0=I$ and $\nabla{I^{t}_{\mathscr{F}}} \approx \nabla{\hat{I}^{t}_{\mathscr{F}}} \defeq (1-\alpha^t ) \nabla{\hat{I}^{t-1}_{\mathscr{F}}} + \alpha^t\nabla{\hat{I}^{t-1}_{\mathscr{F}}}\circ G^{gr}$ with $\nabla{\hat{I}^{1}_{\mathscr{F}}} \defeq (1-\alpha^1)\nabla{I^{0}} + G^{gr}$, where $\circ$ denotes the Hadamard product, $\alpha^t \leq 1$ and $G^{gr}$ aims to enhance important edges and is manually annotated. For ease of explanation, we uniformly use $\nabla{I^{t}_{\mathscr{F}}}$ to represent $\nabla{\hat{I}^{t}_{\mathscr{F}}}$ in the following.\vspace{5pt}\\
	\noindent\textit{Guider consistency} is to enforce the predicted guidance map $G^t$ to be consistent with $I^t_{\mathscr{F}}$ in the gradient domain. The guider consistency loss is as follows:
	\begin{equation}
	\mc{L}^{\mathscr{G}}_{con} \defeq \sum_{t=1}^T \|G^t \circ \overline{\nabla I^t_{\mathscr{F}}}\|_1 + \beta_g\|\overline{G^t} \circ \nabla I^t_{\mathscr{F}}\|_1,
	\label{indi}
	\end{equation}
	where $\overline{\nabla I^t_{\mathscr{F}}}\defeq\bs{1}-\nabla I^t_{\mathscr{F}}$, $\overline{G^t}\defeq\bs{1}-G^t$, $\|\cdot\|_1$ means the $\ell_1$ norm, and $\bs{1}$ denotes the all-one matrix with a compatible size. Further, $\beta_g$ is a constant parameter for balancing the two terms. 
This loss determines if a location of $G^t$ is an edge element or not by comparing the magnitudes of $\overline{\nabla I^t_{\mathscr{F}}}$ and $\beta_g\nabla I^t_{\mathscr{F}}$.\\


	\noindent\textit{Peeler reconstruction} desires the output $I^t$ and $I^{t}_{\mathscr{F}}$ to be as close as possible. A reconstruction loss is adopted:
	\begin{equation}
	\mc{L}^{\mathscr{P}}_{rec} \defeq \sum_{t=1}^T{||I^t -I^{t}_{\mathscr{F}}||_2^2},
	\end{equation}
	where $\|\cdot\|_2$ designates the $\ell_2$ norm.\\

%
    \noindent\textit{Peeler preservation} aims to maintain the structural pixels in $I^t$, corresponding to the pixels with values close to 1 in $G^t$. The gradient responses naturally reflect structure information of an image, thus the peeler preservation loss is defined as the distance between the gradient responses of $I^{t}_{\mathscr{F}}$ and $I^t$ as below:
    	\begin{equation}
	\mc{L}^{\mathscr{P}}_{pre} \defeq \sum_{t=1}^T{||G^t \circ \nabla I^t - G^t \circ \nabla I^{t}_{\mathscr{F}}||_2^2}.
	\end{equation}

    \noindent\textit{Peeler consistency} strictly constrains the peeling process to be in line with the edge map $G^t$. Inspired by \cite{muGIF}, the peeler consistency loss suppresses the gradient magnitude of each pixel in $I^t$ with different strengths. It gives a small penalty on the structural pixels indicated by $G^t$ while a large one on the textural pixels, which is expressed as:
    	\begin{equation}
	\mc{L}^{\mathscr{P}}_{con} \defeq \sum_{t=1}^T{||\frac{\nabla I^t}{\nabla I^{t}_{\mathscr{F}} \circ G^t + \epsilon}||_2^2}.
	\end{equation}
	where $\epsilon=0.005$ is used to avoid division by zero. To make the training procedure more stable and accelerate the convergence speed, the HIPe-Peeler and HIPe-Guider are trained independently. Since the peeler needs to be trained with $G^t$ given, we first learn the guider part by using $\lambda^{\mathscr{G}}_{con}\mc{L}^{\mathscr{G}}_{con}$ only, then freeze the guider and train the peeler using $\mc{L}^{\mathscr{P}}_{rec}+\lambda^{\mathscr{P}}_{pre}\mc{L}^{\mathscr{P}}_{pre}+\lambda^{\mathscr{P}}_{con}\mc{L}^{\mathscr{P}}_{con}$.
%

\section{Experimental Validation}
\label{sec:Exp}

Our model is implemented in PyTorch. All the experiments are carried out on a machine with a GeForce RTX 2080Ti GPU and an Intel Core i7-8700 3.20 GHZ CPU. The optimizer exerts the RMSprop algorithm whose learning rate is set to 0.001 at the beginning and linearly decreases with the increase of epochs. The images are all resized to 256 $\times$ 256 at the training stage and can be any size at the testing. We denote by Ours-S and Ours the supervised ($I^{t}_{\mathscr{F}}$ generated by muGIF \cite{muGIF}) and unsupervised settings in this section, respectively. The weights are set to $\lambda^{\mathscr{G}}_{con} = 1.5$, $\lambda^{\mathscr{P}}_{pre}=0.4$, and $\lambda^{\mathscr{P}}_{con}=4$. The training data for HIPe-Guider are from the BSDS500 dataset \cite{BSDS500}, and those for HIPe-Peeler are the natural images from the RTV \cite{RTV} and ADE20K datasets \cite{ade20k} (3,120 images in total).

\begin{table}[t]
	\centering
	\resizebox{0.48\textwidth}{!}{
		
	\begin{tabular}{c|cccccc}
	\hline
	Method &	\footnotesize{L0}&\footnotesize{RGF}&\footnotesize{SD}&\footnotesize{RTV}&\footnotesize{realLS}&\footnotesize{muGIF}\\
	\hline
	GCC ($\times 10^{-2}$) &	0.49&0.49&0.61&0.48&0.49&0.48\\
	\hline
	\hline
	Method &	\footnotesize{enBF}&\footnotesize{DEAF}&\footnotesize{FIP}&\footnotesize{PIO}&\footnotesize{Ours-S}&\footnotesize{Ours}\\
	\hline
	GCC ($\times 10^{-2}$)&	0.50&0.46&0.48&0.47&0.46&\textbf{0.45}\\
	\hline
\end{tabular}
}

	\caption{Quantitative comparison in GCC. The smoothing degrees are controlled around a similar level for aligning different methods. The best  results are highlighted in bold. Lower GCC values indicate better performance.}
	\vspace{5pt}
	\label{corr}

		\resizebox{0.48\textwidth}{!}{
	\begin{tabular}{c|cccccc}
		\hline
	Method &	\footnotesize{L0}&\footnotesize{RGF}&\footnotesize{SD}&\footnotesize{RTV}&\footnotesize{L1}&\footnotesize{realLS}\\
		\hline
	Time &	{6.36}&{3.16}&{53.27}&{9.31}&{913.17}&{1.75}\\
		\hline
		\hline
	Method &	\footnotesize{muGIF}&\footnotesize{enBF}&\footnotesize{DEAF}&\footnotesize{FIP}&\footnotesize{PIO}&\footnotesize{Ours}\\
		\hline
	Time &	{11.12}&{\textit{1.21}}&{5.13}$^\dagger$&{0.034}$^\dagger$&\textbf{0.008}$^\dagger$&{\textbf{0.41}/\textit{0.016}$^\dagger$}\\
		\hline
	\end{tabular}
}
\caption{Runtime comparison on processing a 1080p image (1627 $\times$ 1080) \textit{in seconds}. The CPU times are unmarked and GPU times are marked by  $\dagger$, respectively.}
	\label{speed}
\end{table}

\subsection{Peeler Evaluation}
This part is to test the abilities of different methods in terms of edge preserving filtering.
Traditional methods including L0 \cite{L0}, RTV \cite{RTV}, RGF \cite{RGF}, SD \cite{2017Robust}, L1 \cite{2015L1}, muGIF \cite{muGIF}, realLS \cite{RWLS}, and enBF \cite{LiuZCSHY20}, as well as deep learning approaches including DEAF \cite{DCNN}, FIP \cite{2017Fast}, and PIO \cite{fan2019decouple} are involved in comparisons. The codes of competitors we use are all provided by the authors. To measure the image quality, the gradient correlation coefficient, {\it i.e.} $\text{GCC}(P,I)\defeq\text{mean}(\|\nabla P\circ\nabla I\|_1$),  is adopted, indicating the uncorrelation in the gradient domain and theoretically supported by \cite{2017Exclusivity}. The evaluation for GCC is performed on the 200 images from the test set of BSDS500 \cite{UCM}. In addition, the running time is considered to reflect the efficiency.

\begin{figure}[t]
	\centering
	\includegraphics[width=0.95\linewidth]{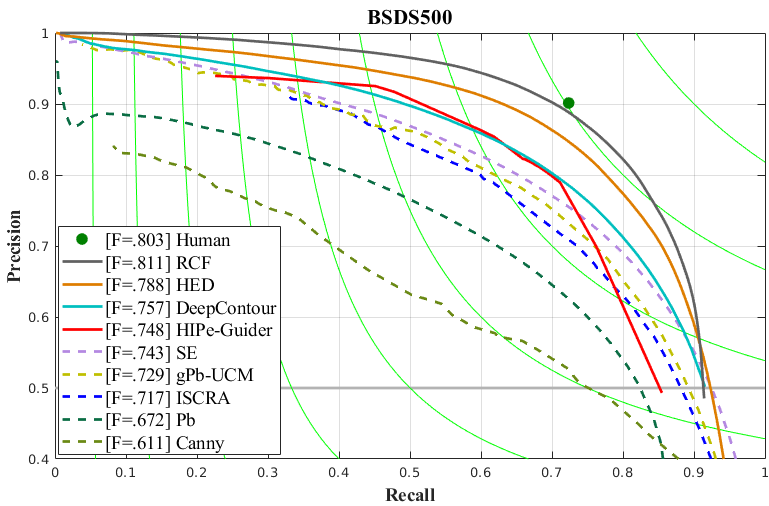}
	\caption{Edge detection comparison in terms of precision-recall curve on the BSDS500 dataset.} 
	
	\label{roc}
	\vspace{3pt}
	\centering
	\begin{subfigure}{0.32\linewidth}
		\includegraphics[width=1\linewidth]{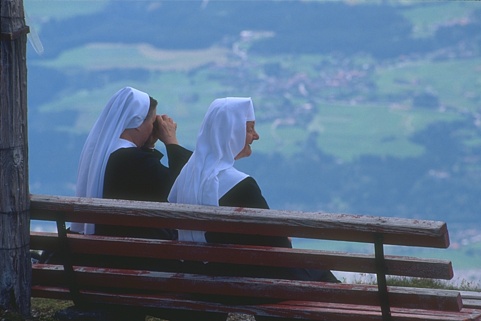}
		\subcaption{Input}
	\end{subfigure}
	\begin{subfigure}{0.32\linewidth}
		\includegraphics[width=1\linewidth]{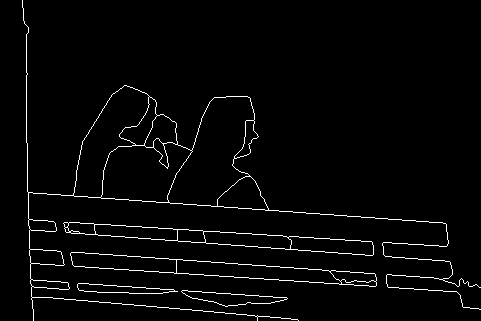}
		\subcaption{GT edge}
	\end{subfigure}
	\begin{subfigure}{0.32\linewidth}
		\includegraphics[width=1\linewidth]{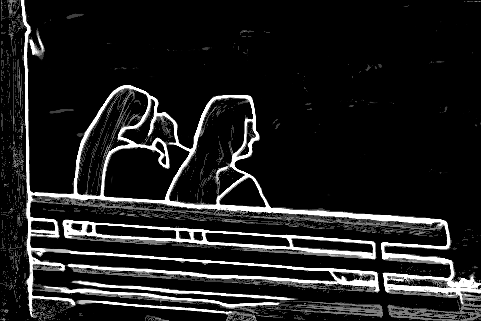}
		\subcaption{Our result}
	\end{subfigure}
	\vspace{-5pt}
	\caption{An example of our edge confidence map.}
	\vspace{-10pt}
	\label{edge_confidence}
\end{figure}

We tweak the hyper-parameters of the competitors to reach a similar smoothing degree for fairness. Table \ref{corr} reports the numerical results in GCC, and Figure \ref{com} depicts a visual comparison among the competitors. As can be seen from the numbers, our HIPe takes the first place in terms of GCC, indicating that $\nabla I$ and $\nabla P$ considerably satisfy the orthogonality. As for the visual effect, we observe that the visual quality of the results by L0, RGF and PIO is relatively poor when the smoothing degree increases, and PIO suffers from noticeable color shift issue. Though performing remarkably better, RTV and muGIF fail to completely smooth or preserve some regions, like the right corner of first case, and the border of flowers. In comparison, our approach can achieve visually pleasing filtering results in terms of region flattening and edge preserving, thanks to the ability of our peeler to strictly adhere to the unambiguous edge guidance.

Additionally, the time costs of different methods at inference stage are displayed in Table \ref{speed}. Clearly, our method is much more efficient than the traditional methods even on a CPU. The HIPe-Peeler and PIO can reach super real-time speeds to handle 1080p images on a 2080Ti GPU. Moreover, both the traditional and deep models, except for our HIPe, can hardly process images with spatially-variant scales and/or user provided/edited guidance maps, as shown in Fig. \ref{variededge}, which would remarkably broaden the applicability of filtering/peeling. By simultaneously considering the peeling quality, efficiency and flexibility, our method is among the most attractive choices for practical use.

\subsection{Guider Evaluation}
This section evaluates the ability of guidance prediction. Due to the scale-space nature, our HIPe-Guider can easily construct an edge confidence map for an image through combining the edge maps obtained from different recurrent steps. Specifically, we train the HIPe-Guider using annotated edges in the BSDS500 dataset as $G^{gr}$ and recursively run it for 24 rounds. The maps are simply averaged, which are all processed by non maximum suppression. An example is shown in Figure \ref{edge_confidence}. We compare the results of edge detection on the BSDS500 dataset using precision-recall curve and F-score. The competing candidates are non-deep methods, including Canny \cite{canny}, Pb \cite{pb}, gPb-UCM \cite{UCM}, ISCRA \cite{ISCRA}, and SE \cite{SE}, and deep methods, including DeepContour \cite{deepcontour}, HED \cite{HED} and  RCF \cite{RCF}. As shown in Figure \ref{roc}, our method obtains the ODS-F (optimal dataset scale F-score) of 0.748, being obviously superior to Canny and Pb, competitive with DeepContour and even slightly better than ISCRA, gPb-USM and SE. Compared with the deep supervised HED and RCF, there indeed exists a margin. This is reasonable because \emph{we do not introduce any pre-trained classification network as backbone. Our HIPe-Guider (998KB) is significantly smaller than DeepContour (27.5Mb), HED (56.1Mb) and RCF (56.5Mb)}.

\begin{figure}
	\begin{subfigure}{0.24\linewidth}
		\includegraphics[width=1\linewidth]{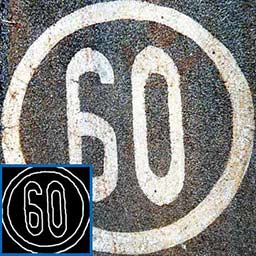}
        \includegraphics[width=1\linewidth]{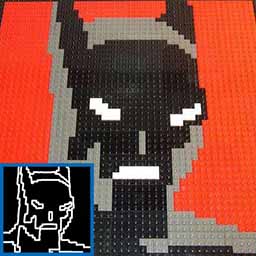}
		\subcaption{Input/Edge}
	\end{subfigure}
	\begin{subfigure}{0.24\linewidth}
		\includegraphics[width=1\linewidth]{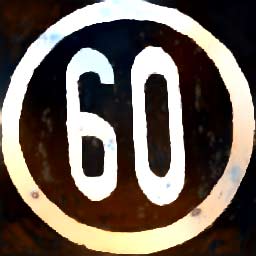}
        \includegraphics[width=1\linewidth]{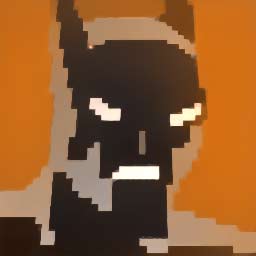}
		\subcaption{w/o Resi.}
	\end{subfigure}
	\begin{subfigure}{0.24\linewidth}
		\includegraphics[width=1\linewidth]{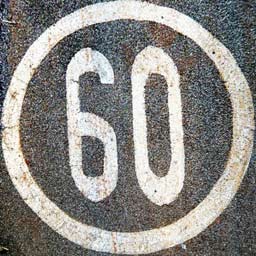}
        \includegraphics[width=1\linewidth]{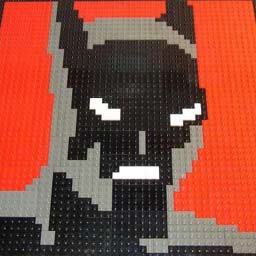}
		\subcaption{ w/o $\mc{L}^{\mathscr{P}}_{con}$}
	\end{subfigure}
	\begin{subfigure}{0.24\linewidth}
		\includegraphics[width=1\linewidth]{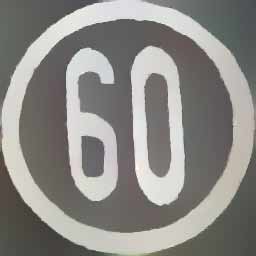}
        \includegraphics[width=1\linewidth]{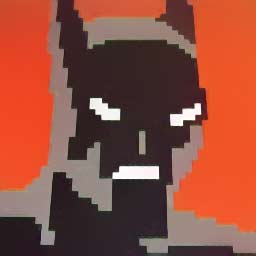}
		\subcaption{Ours}
	\end{subfigure}
	
	\vspace{-5pt}
	\caption{Effectiveness analysis on different strategies.}
	\vspace{-7pt}
	\label{ablation}
\end{figure}

%
%

\subsection{Other Issues}
\textbf{Ablation analysis.}
The edge predictor in HIPe-Guider only involves the guider consistency loss $\mc{L}^{\mathscr{G}}_{con}$. Thus, we omit the ablation analysis from the loss perspective. As for the peeling manner, the peeler has two possible ways to map the input $I^{t-1}$ to one component. In this work, we incline to the mapping of $C_{t}\leftarrow\mathscr{P}(I^{t-1}, G^{t})$. The reason is that the information of $C_{t}$ is less and distributed simpler to be quickly recovered than that of $I^t$, which is corroborated by Figure \ref{ablation} (b) and (d). As can be clearly viewed, the result by directly mapping to $I^{t}$ as shown in Figure \ref{ablation} (b) suffers from the issue of color attenuation. Further, we test the effect of removing the peeler consistency loss $\mc{L}^{\mathscr{P}}_{con}$ . The corresponding results are given in Figure \ref{ablation} (c), revealing that the textures cannot be smoothed out without adopting $\mc{L}^{\mathscr{P}}_{con}$. The peeler for ablation analysis is trained in an unsupervised manner.\\

\begin{table}[t]
	\centering
	\resizebox{0.48\textwidth}{!}{
		
	\begin{tabular}{c|ccccc}
	\hline
	  &	ECSSD &PASCAL-S &HKU-IS  &SOD  &DUTS-TE \\
	\hline
    \hline
    HS \cite{hsaliency} &0.226&0.260&0.213&0.280&0.241\\

    HS \cite{hsaliency} + Ours &0.128&0.197 &0.123&0.209 &0.159\\
    \hline
    PoolNet \cite{poolnet} &0.039&0.074&0.032&0.100&0.039\\

    PoolNet \cite{poolnet} + Ours &0.034&0.071&0.026&0.096&0.036\\
    \hline
	CSF \cite{sal100K} &	0.033&0.068&0.030&0.098&0.037\\

	CSF \cite{sal100K} + Ours &	0.028&0.065&0.025&0.094&0.034\\
	\hline
	EGNet \cite{EGnet} &	0.040&0.074&0.031&0.097&0.039\\

    EGNet \cite{EGnet} + Ours &	0.033&0.071&0.026&0.097&0.036\\
	\hline
\end{tabular}
}
%
%
%
%
%

\caption{ Quantitative comparison in MAE. The lower the MAE, the better. \textit{Name + Ours} refers to predict from the saliency maps generated by \textit{Name}.}
	\label{saliency}
\end{table}

%
%
%
%
%
%

\begin{figure}
\centering
	\begin{subfigure}{0.48\linewidth}
		\includegraphics[width=1\linewidth]{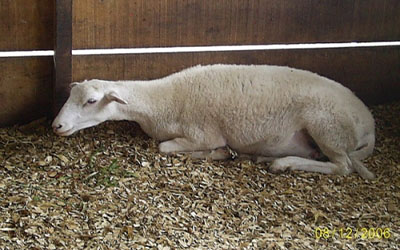}
		\subcaption{Input}
	\end{subfigure}
	\begin{subfigure}{0.48\linewidth}
		\includegraphics[width=1\linewidth]{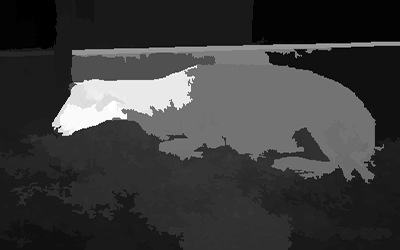}
		\subcaption{HS}
	\end{subfigure}
	\begin{subfigure}{0.48\linewidth}
		\includegraphics[width=1\linewidth]{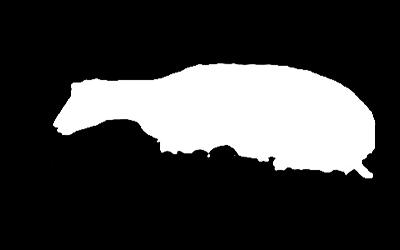}
		\subcaption{ HS+Ours}
	\end{subfigure}
		\begin{subfigure}{0.48\linewidth}
		\includegraphics[width=1\linewidth]{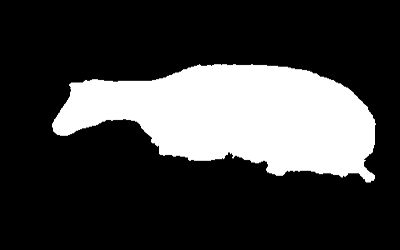}
		\subcaption{GT saliency}
	\end{subfigure}
	
	\caption{Visual comparison between HS and HS + Ours.}
	\vspace{-12pt}
	\label{saliency_fig}
\end{figure}

\textbf{Applications.}
For fully utilizing our multi-scale nature, we propose a novel strategy to apply our framework to the task of saliency detection. We first utilize an existing model to separately perform saliency detection on a raw image and its four filtering results (5 in total) generated by our framework, then train a lightweight network (merely 91KB) on the DUTS-TR dataset \cite{DUTS} to predict a better saliency map from these five saliency detection results. We follow \cite{SF, sal100K, EGnet,poolnet} to use $\text{MAE}(S^o, S^{gt}) := \text{mean}(| S^o - S^{gt}|)$ as the evaluation matric, and evaluate our method on public datasets: ECSSD \cite{hsaliency}, PASCAL-S \cite{PASCAL}, HKU-IS \cite{HKU-IS}, SOD \cite{SOD}, and DUTS-TE \cite{DUTS}. The quantitative and qualitative results are shown in Table \ref{saliency} and Figure \ref{saliency_fig}, respectively, proving that our framework is effective for the salient object detection since some useful features may be prominent at different scales. Single scale smoothing may not produce better results, \textit{e.g.} some salient structures might be over-smoothed in large-scales. Note that for all saliency detection methods, the input is identical and no extra information is used. The proposed image diasssenbly/peeling produces several levels for capturing different cues.

Besides saliency detection, our framework is flexible to promote the performance of many other tasks, such as image enhancement \cite{MSR}, image editing \cite{abstraction}, mutual structure discovery \cite{JFMS} and image segmentation \cite{UCM}. For example, a multi-scale Retinex based low-light enhancement case is provided in Figure \ref{app} (a-c) by substituting HIPe for the Gaussian filter in \cite{MSR}, and a multi-scale image abstraction \cite{abstraction} example in (d-f). Figure \ref{app} (g-i) show that we allow the reference (depth image for generating edge map) and target (RGB) images from different sources.

\begin{figure}
	\centering
	\begin{subfigure}{0.32\linewidth}
\includegraphics[width=1\linewidth]{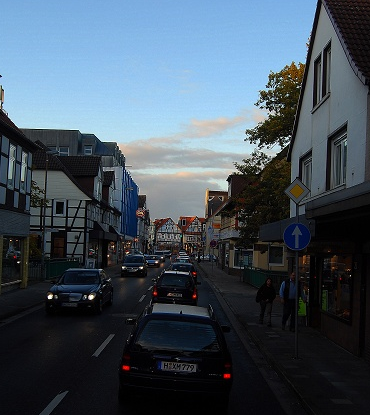}
\subcaption{Input}
\end{subfigure}
\begin{subfigure}{0.32\linewidth}
\includegraphics[width=1\linewidth]{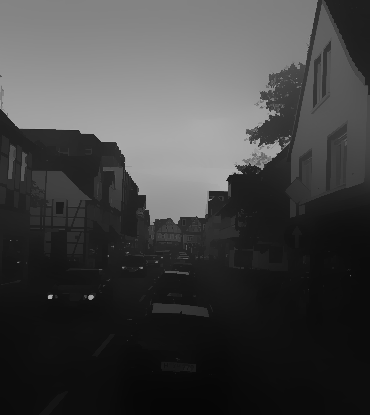}
\subcaption{Illumination}
\end{subfigure}
\begin{subfigure}{0.32\linewidth}
\includegraphics[width=1\linewidth]{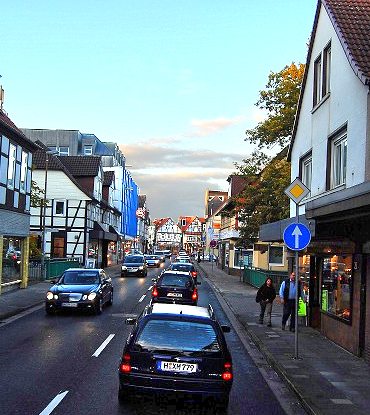}
\subcaption{Enhanced result}
\end{subfigure}\\
	\begin{subfigure}{0.32\linewidth}
	\includegraphics[width=1\linewidth]{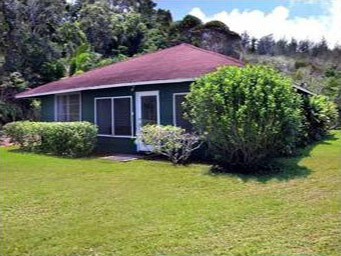}
	\subcaption{Input}
\end{subfigure}
\begin{subfigure}{0.32\linewidth}
	\includegraphics[width=1\linewidth]{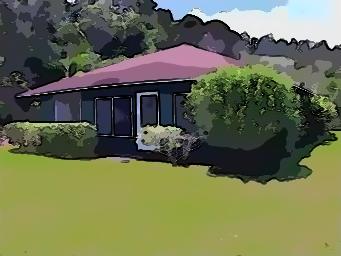}
	\subcaption{Abstraction 1}
\end{subfigure}
\begin{subfigure}{0.32\linewidth}
	\includegraphics[width=1\linewidth]{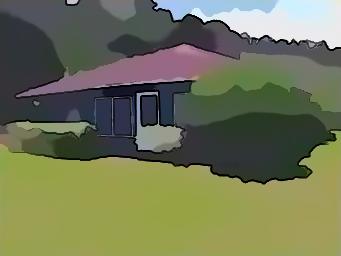}
	\subcaption{Abstraction 2}
\end{subfigure}\\
\begin{subfigure}{0.32\linewidth}
\includegraphics[width=1\linewidth]{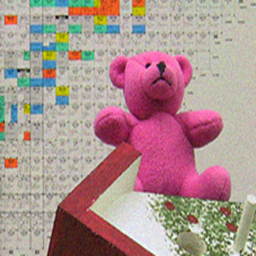}
\subcaption{Color input}
\end{subfigure}
\begin{subfigure}{0.32\linewidth}
\includegraphics[width=1\linewidth]{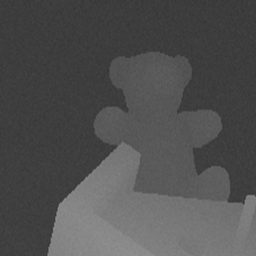}
\subcaption{Depth reference}
\end{subfigure}
\begin{subfigure}{0.32\linewidth}
\includegraphics[width=1\linewidth]{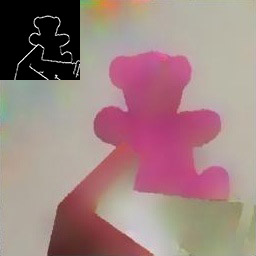}
\subcaption{Smoothed result}
\end{subfigure}

\caption{The three rows correspond to the applications of low-light enhancement, image abstraction, and depth-guided RGB image filtering, respectively.}
\vspace{-15pt}
\label{app}
\end{figure}
\section{Concluding Remarks}
This paper has proposed a modern framework for hierarchically organizing images. A flexible and compact recurrent network, namely hierarchical image peeling net, has been developed based on theoretical analysis to efficiently and effectively fulfill the task, which jointly takes into account the peeling hierarchy, structure preservation, flexibility, and model efficiency, making it attractive for practical use. The network  can be trained in both supervised and unsupervised manners. Experimental results have been provided to demonstrate the advantages of our design. Our framework also has much potential to derive new applications and inspire new technical lines to solving existing problems, such as object detection and image segmentation.
{\small
\bibliographystyle{ieee_fullname}
\bibliography{reference2}
}

\end{document}